\documentclass[conference]{IEEEtran}

\IEEEoverridecommandlockouts                              

\usepackage{graphics}
\usepackage{epsfig}
\usepackage{mathptmx} 
\usepackage{mathrsfs}
\usepackage{times}
\usepackage{amsmath}

\usepackage{amsthm}
\usepackage{amssymb}
\usepackage{wasysym}
\usepackage{txfonts}
\usepackage{bbm}
\usepackage[singlelinecheck=false,font=footnotesize]{caption}
\usepackage{subcaption}
\usepackage{color}
\usepackage[dvipsnames]{xcolor}
\usepackage{indentfirst}
\usepackage{multirow}

\usepackage{color, colortbl}
\usepackage{algorithm}
\usepackage{algpseudocode}

\usepackage[
    style=ieee,
    doi=false,
    isbn=false,
    url=false,
    eprint=false,
    backend=bibtex,
    natbib=true
    ]{biblatex}

\graphicspath{{Images/}}

\bibliography{references}

\newtheorem{defn}{Definition}
\newtheorem{rem}[defn]{Remark}
\newtheorem{lem}[defn]{Lemma}

\newtheorem{assum}[defn]{Assumption}

\newtheorem{thm}[defn]{Theorem}

\providecommand{\R}{\ensuremath \mathbb{R}}

\providecommand{\T}{\ensuremath T}
\providecommand{\tfin}{t_\regtext{f}}

\renewcommand{\P}{\ensuremath \mathcal{P}}

\newcommand{\Lfi}{\mathcal{L}_{f_i}}
\newcommand{\Lgi}{\mathcal{L}_{g_i}}
\newcommand{\norm}[1]{\left\Vert#1\right\Vert}

\newcommand{\defemph}[1]{\emph{#1}}

\newcommand{\vep}{\varepsilon}

\newcommand{\inv}{^{-1}}
 
\newcommand{\Z}{\mathbb{Z}}

\newcommand{\bd}[1]{\partial #1}
\newcommand{\ceil}[1]{\left\lceil#1\right\rceil}

\newcommand{\regtext}[1]{\mathrm{\textnormal{#1}}}

\newcommand{\hi}{_\regtext{hi}}
\newcommand{\hio}{_{\regtext{hi},0}}

\newcommand{\hij}{_{\regtext{hi},j}}
\newcommand{\xhij}{x_{\regtext{hi},j}}
\newcommand{\xhijp}{x_{\regtext{hi},j+1}}

\newcommand{\NAFexact}{{\varphi}}
\newcommand{\NAF}{{\tilde{\NAFexact}}}
\newcommand{\FRS}{F}
\newcommand{\plan}{_\regtext{plan}}
\newcommand{\sense}{_\regtext{sense}}
\newcommand{\brk}{_\regtext{brake}}

\newcommand{\brko}{_{\regtext{brake},0}}

\newcommand{\move}{_\regtext{move}}

\newcommand{\moveo}{_{\regtext{move},0}}
\newcommand{\des}{_\regtext{des}}

\newcommand{\disc}{_\regtext{disc}}
\newcommand{\discmax}{_\regtext{disc,max}}
\newcommand{\discfn}{\regtext{\texttt{disc}}}
\newcommand{\sample}{\regtext{\texttt{sample}}}
\newcommand{\obs}{_\regtext{obs}}

\newcommand{\obsmax}{_\regtext{obs,max}}
\newcommand{\stp}{_\regtext{stop}}

\newcommand{\stpo}{_{\regtext{stop},0}}

\newcommand{\naf}{_\regtext{NAF}}
\newcommand{\pred}{_\regtext{pred}}
\newcommand{\rel}{_\regtext{rel}}

\newcommand{\idx}{\regtext{proj}_X}

\newcommand{\vmax}{{v_\regtext{max}}}

\newcommand{\bt}{\beta}
\newcommand{\dl}{\delta}

\newcommand{\om}{\omega}

\definecolor{Gray}{gray}{0.9}
\newcolumntype{g}{>{\columncolor{Gray}}c}

\newcommand{\comp}{^{\regtext{C}}}

\usepackage{hyperref}

\title{\LARGE \bf Towards Provably Not-at-Fault Control of Autonomous Robots in Arbitrary Dynamic Environments}
\author{Sean Vaskov*$^1$, Shreyas Kousik*$^1$, Hannah Larson$^1$, Fan Bu$^1$, James Ward$^2$,\\ Stewart Worrall$^2$, Matthew Johnson-Roberson$^3$, Ram Vasudevan$^1$
\thanks{This work is supported by the Ford Motor Company via the Ford-UM Alliance under award N022977, and the Office of Naval Research under award number N00014-18-1-2575.}
\thanks{* These authors contributed equally to this work.}
\thanks{$^{1}$Mechanical Engineering, University of Michigan, Ann Arbor, MI {\tt\small <skvaskov,skousik,hmlarson,fanbu,ramv>@umich.edu}}%
\thanks{$^{2}$Australian Centre for Field Robotics, University of Sydney, New South Wales, Australia {\tt\small <j.ward,s.worrall>@acfr.usyd.edu.au}}%
\thanks{$^{3}$Naval Architecture and Marine Engineering, University of Michigan, Ann Arbor, MI {\tt\small <mattjr>@umich.edu}}%
}
\begin{document}

\maketitle
\thispagestyle{empty}
\pagestyle{plain}

\begin{abstract}
As autonomous robots increasingly become part of daily life, they will often encounter dynamic environments while only having limited information about their surroundings.
Unfortunately, due to the possible presence of malicious dynamic actors, it is infeasible to develop an algorithm that can guarantee collision-free operation.
Instead, one can attempt to design a control technique that guarantees the robot is not-at-fault in any collision.
In the literature, making such guarantees in real time has been restricted to static environments or specific dynamic models.
To ensure not-at-fault behavior, a robot must first correctly sense and predict the world around it within some sufficiently large sensor horizon (the prediction problem), then correctly control relative to the predictions (the control problem).
This paper addresses the control problem by proposing Reachability-based Trajectory Design for Dynamic environments (RTD-D), which guarantees that a robot with an arbitrary nonlinear dynamic model correctly responds to predictions in arbitrary dynamic environments.
RTD-D first computes a Forward Reachable Set (FRS) offline of the robot tracking parameterized desired trajectories that include fail-safe maneuvers.
Then, for online receding-horizon planning, the method provides a way to discretize predictions of an arbitrary dynamic environment to enable real-time collision checking.
The FRS is used to map these discretized predictions to trajectories that the robot can track while provably not-at-fault.
One such trajectory is chosen at each iteration, or the robot executes the fail-safe maneuver from its previous trajectory which is guaranteed to be not at fault.
RTD-D is shown to produce not-at-fault behavior over thousands of simulations and several real-world hardware demonstrations on two robots: a Segway, and a small electric vehicle.
\end{abstract}

\section{Introduction}\label{sec:introduction}

Autonomous ground robots, such as autonomous cars, have the potential to increase people's mobility and the accessibility of services.
This requires them to operate in environments alongside humans or other surrounding actors that may be moving.
Since a robot's sensors can only provide information in a finite neighborhood around it, robots typically operate using a receding-horizon strategy, in which new control inputs are computed as the previous ones are executed.
Most autonomous mobile robots generate these control inputs using a three-level hierarchy to enable real-time performance \cite{katrakazas2015_motionplanning,mcnaughton_thesis_2011,boss2008urbanchallenge}.
At the top of the hierarchy, a high-level planner generates a coarse task description, such as GPS waypoints for an autonomous car to follow.
A mid-level planner then generates a reference trajectory that attempts to execute the high-level task.
Finally, a low-level controller (e.g., a proportional or model-predictive controller) attempts to track the reference trajectory by actuating the robot.
To operate in real time, the high-level planner typically does not consider the robot's dynamics, and the low-level controller typically does not consider the robot's surroundings.
Therefore, the mid-level planner must generate a reference trajectory that, when tracked by the low-level controller, causes the robot to avoid obstacles, making the mid-level controller responsible for ensuring safety.

\begin{figure}[t]
    \centering
    \includegraphics[width=0.98\columnwidth]{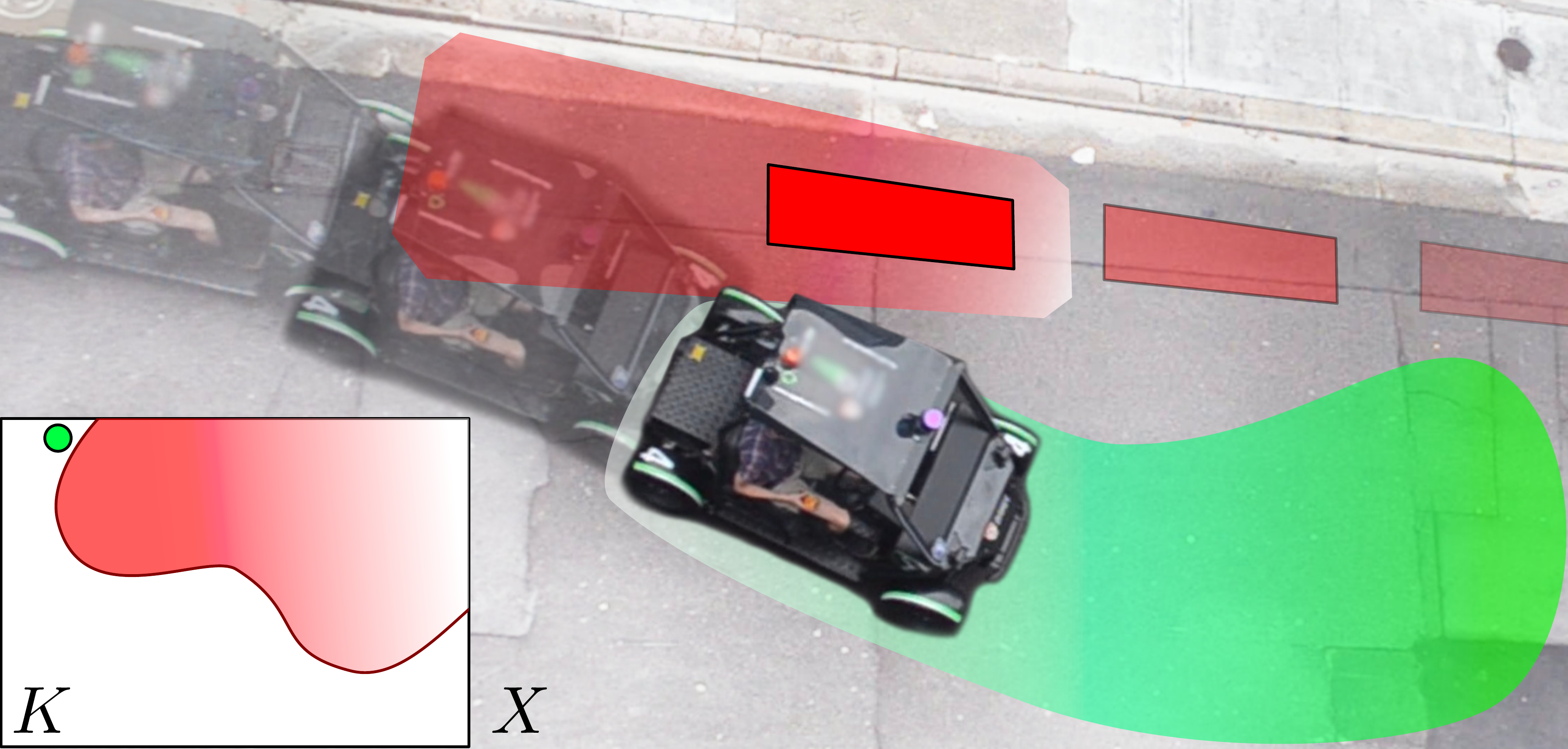}
    \caption{The proposed method planning a guaranteed not-at-fault trajectory for the EV robot (moving from left to right) around a dynamic obstacle (in red, moving from right to left) in the plane $X$; opacity increases with time.
    At the last depicted time instance, the obstacle's predicted motion fades from white to red, and the forward reachable set of the EV fades from white to green.
    In the trajectory parameter space $K$, the planned trajectory is a green point lying outside the parameters for which the robot could be at-fault in a collision.
    At runtime, the proposed method conservatively approximates the set of not-at-fault trajectories by identifying the set of trajectories that would intersect with a discretized representation of the obstacle and its prediction. 
    This paper proves that this obstacle representation, which enables real-time planning performance, is sufficient to ensure collision free behavior.   
    Videos are available of the EV (\texttt{\url{www.roahmlab.com/ev\_dyn\_obs\_demo}}) and Segway (\texttt{\url{www.roahmlab.com/segway\_dyn\_obs\_demo}}) robots.}
    \label{fig:ev_birdseye_timelapse}
    \vspace*{-0.5cm}
\end{figure}

Unfortunately, guaranteeing safe operation in arbitrary scenarios is intractable.
Consider a vehicle on a highway, surrounded by other cars driving at the same speed.
In this instance, any surrounding vehicle could act maliciously to cause a collision; nevertheless, it is still possible to assign fault \cite{shwartz2017_mobileye_safety,yuxiao2018_polar}.
As a result, safety is more appropriately defined as the robot being \defemph{not-at-fault} for a collision.

As depicted in Figure \ref{fig:ev_birdseye_timelapse}, this paper presents a mid-level planner that generates provably not-at-fault trajectories in real time.
Note, this work is not concerned with how to sense or predict obstacles in the robot's surroundings.
These problems, difficult in their own right, are the subject of ongoing research \cite{hoj2017_ped_pred,bajcsy2018_multi_robot_collision_avoidance,liu2017_dynamic_env_kinematic}.
Note that predictions can be made more conservative by increasing the uncertainty associated with observations, at the expense of reducing free space for planning.
However, to the best of our knowledge, even when obstacles are sensed and predicted conservatively, no numerical method has yet been shown to guarantee not-at-fault, real-time creation of reference trajectories with respect to such information.

\subsection{Related Work}

To guarantee that the trajectories they design are not-at-fault, mid-level planners must perform both \emph{planning} (the creation of a trajectory for the robot to track) and \emph{validation} (checking that the robot satisfies environment and state constraints).
If the planner fails validation in one iteration, the robot can attempt to execute a not-at-fault fail-safe maneuver created by the planner in a previous iteration.
Depending upon how they plan and validate, existing mid-level planners can be broadly divided into three categories: check, correct, or select.

``Check'' methods use precomputed reference trajectories that include fail-safe maneuvers and are checked for collisions online.
For example, \citet{mcnaughton_thesis_2011} use a state lattice to generate trajectories, and check collisions with respect to an occupation grid at a discrete number of points; however, this check does not guarantee that the whole trajectory is not-at-fault.
Zonotope reachability methods, on the other hand, check whether an entire trajectory intersects with any obstacles in the environment \cite{althoff2014online,liu2017_dynamic_env_kinematic}.
This requires a reachability computation for a high-dimensional system at run time, which can be challenging to perform in real-time in arbitrary scenarios.

``Correct'' approaches generate a kinematically-feasible reference trajectory, then modify the control inputs to ensure the robot is not-at-fault when tracking the reference.
For instance, one can compute a lookup table of control inputs that combat tracking error with Hamilton-Jacobi (HJ) reachability analysis \cite{herbert2017fastrack,bajcsy2018_multi_robot_collision_avoidance}; however, since the level set method \cite{mitchell_toolboxLS} used to measure tracking error does not necessarily generate an outer approximation to the reachable set, it is unclear how to certify that this approach correctly responds to tracking error that may lead to a collision \cite[Section III-A]{mitchell2005time}.
Another ``correct'' approach computes a Control Barrier Function (CBF) that is similar to a Lyapunov function over a continuous control input space; this has been applied successfully to active cruise control and lane keeping \cite{xu2016correctness}, and to low-speed robots that can treat dynamics as a disturbance in an off-line fashion \cite{yuxiao2018_polar}.
However, it is unclear how to extend this approach to fast nonlinear systems in arbitrary environments under real-time constraints \cite{borrman2015_aaron_ames_swarm_cbf}.

``Select'' approaches create a set of reference trajectories offline, and select one that is not-at-fault at each planning iteration.
For instance, \citet{majumdar2016funnel} precompute a finite set of ``funnels,'' which are volumes in state space that contain reference trajectories and associated tracking error, using Sums-of-Squares (SOS) programming; at runtime, they propose to optimize only over those funnels that do not intersect with sensed obstacles.
However, the Bullet Graphics Engine \cite{coumans2013bullet}, which they apply to check for collisions between the funnel and obstacle at run-time, is unable to certify that it detects a collision if one exists \cite{sagardia2014new}.
To avoid using a finite set of reference trajectories, one can precompute a Forward Reachable Set (FRS) over a continuous, parameterized trajectory space \cite{kousik2017safe,kousik2018_RTD_ijrr}.
For online planning, \citet{kousik2018_RTD_ijrr} prescribe a numerical method to certify that a trajectory is collision free by verifying that the FRS of that trajectory does not intersect with a discretized obstacle representation.
Unfortunately, ``select'' methods have only been developed for static environments; furthermore, they implicitly require that a fail-safe maneuver can be performed.

\subsection{Contribution}

This paper presents a novel ``select'' method for mid-level planning called Reachability-based Trajectory Design for Dynamic environments (RTD-D).
To the best of our knowledge, this is the first real-time, mid-level planner that is certified to generate not-at-fault, dynamically-feasible trajectories in arbitrary dynamic environments.
The contributions of this paper are four-fold.
First, we formulate a minimum sensor horizon requirement for planning in dynamic environments to ensure not-fault behavior (Section \ref{sec:dynamic_environments}).
Second, we formulate an offline FRS computation that explicitly includes a fail-safe maneuver (Section \ref{sec:reachability}).
Third, we prescribe a method for discretizing obstacle predictions in space and time that enables real-time operation while guaranteeing collision-free behavior (Section \ref{sec:not-at-fault_plans}).
Fourth, we confirm that RTD-D is provably not-at-fault over thousands of simulations and compare its performance to a state lattice planner; and we show that RTD-D is effective in the real-world on two hardware platforms: a Segway and an Electric Vehicle (EV), shown in Figure \ref{fig:ev_and_segway} (Section \ref{sec:demos}).
The rest of the paper is organized as follows: the end of this section presents notation; Section \ref{sec:models} introduces the dynamic models used for planning; and Section \ref{sec:conclusion} provides concluding remarks.

\subsection{Notation}\label{subsec:notation}

The complement of a set $A$ is $A\comp$.
The power set of $A$ is $\P(A)$.
The set of continuous (resp. $n$-times differentiable), scalar-valued functions with domain $A$ is $C(A)$ (resp. $C^n(A)$).
The support of a function is $\regtext{supp}(\cdot)$.
The operator $\ceil{\cdot}: \R \to \Z$ rounds up to the nearest integer.
The Hadamard (elementwise) product is denoted by $\circ$.
The set $L_d(\T)$ is the space of absolutely integrable functions from the set $\T$ to $[-1,1]^2$.
\section{Dynamic Models}\label{sec:models}

This paper proposes a receding-horizon planning algorithm that constructs a new trajectory to track while following the trajectory designed during the previous planning iteration.
To construct a new trajectory in each iteration, the planner must be able to estimate the future position of the robot while it follows the previously-constructed trajectory.
This is accomplished using a high-fidelity model.
Since this model may be complex, it may be prohibitive to use in real-time optimization for trajectory design.
As a result, the planner requires a simplified description of the robot.
We refer to this model as the trajectory-producing model. 
This section presents this pair of models and explains how they are used online.

\subsection{High-Fidelity Model}

We estimate the future position of the robot using a \defemph{high-fidelity model} $f\hi: \T\times X\hi\times U \to \R^{n\hi}$ for which
\begin{align}
    \dot{x}\hi(t) = f\hi(t,x\hi(t),u(t)),\label{eq:high-fidelity_model}
\end{align}
where time $t$ is in the \defemph{planning time horizon} $\T = [t_0,\tfin]$.
The state $x\hi$ is in the space $X\hi \subset \R^{n\hi}$, and inputs are drawn from $U \subset \R^{n_U}$.
Since planning is done in a receding-horizon fashion, without loss of generality (WLOG), let each planned trajectory (i.e., each planning iteration) begin at $t_0 = 0$.
In addition, we assume that robot's speed is bounded:
\begin{assum}\label{ass:speed}
The robot has a maximum speed $\vmax$.
\end{assum}

\noindent We assume that the difference between the true state of the robot and the future state estimate under \eqref{eq:high-fidelity_model} beginning from a measured initial state satisfies the following assumption:
\begin{assum}\label{ass:state_estimation_within_epsilon}
Suppose for some $t \in \T$, $x\hi(t)$ is the future state estimate computed by forward integrating the high-fidelity model \eqref{eq:high-fidelity_model} from a measured initial condition. 
Then, the absolute difference between $x\hi(t)$ and the true state of the robot in each coordinate is bounded by $\vep_i > 0$ for each $i \in \{ 1,\cdots,n\hi\}$ and for all $t \in \T$.
\end{assum}

\noindent Since our focus is on planning for ground vehicles, we make the following assumption.

\begin{assum}\label{ass:planar_rigid_body}
The robot operates in the plane.
Define $X \subset \R^2$ as the \defemph{spatial coordinates} of the robot's body such that $X \subset X\hi$.
We denote these coordinates as $x = (x_1, x_2) \in X$.
The operator $\idx: X\hi \to X$ projects points in $X\hi$ to $X$ via the identity relation.
The robot is a rigid body that lies in the compact, convex set $X_0 \subset X$ of initial conditions at $t = 0$; we call $X_0$ the \defemph{robot footprint}.
\end{assum}

\noindent The following definition summarizes prediction error and is used to buffer obstacles as described in Section \ref{sec:dynamic_environments}.

\begin{defn}\label{def:spatial_est_error}
The robot's \defemph{maximum spatial estimation error} is $\vep = (\vep_1^2 + \vep_2^2)^{1/2}$ where $\vep_1, \vep_2$ are the error in $x_1, x_2$ as in Assumption \ref{ass:state_estimation_within_epsilon}, 
\end{defn}

\subsection{Desired Trajectories}\label{subsec:desired_trajectories}

Since we focus on real-time planning, we make the following assumption.
\begin{assum}\label{ass:tau_plan_and_braking}
During each planning iteration, the robot has $\tau\plan > 0$ amount of time to pick a new input.
If the robot cannot find a new input in a planning iteration, it begins a ``fail-safe'' maneuver.
In this work, the fail-safe maneuver is braking to a stop; the robot stays stopped until a new input is found.
\end{assum}

We use the following trajectory-producing model with dynamics $f: \T\times X\times K \to \R^2$ to enable real-time planning.
\begin{defn}\label{def:time_phases}
Let $\T = \T\move\cup\T\brk(k)\cup\T\stp(k)$.
We call $\T\move := [0,\tau\plan]$ the \defemph{moving phase}; $\T\brk(k) := [\tau\plan, \tau\plan + \tau\brk(k)]$ the \defemph{braking phase}, and $\T\stp(k) := [\tau\plan + \tau\brk(k), t_f]$ the \defemph{stopped phase}.
The function $\tau\brk: K \to \R_{\geq 0}$ is the \defemph{braking time} of each desired trajectory.

The \defemph{trajectory-producing model} is then written
\begin{align}\label{eq:traj-prod_model_time_phases}
\dot{x}(t) = f(t,x,k) &= \begin{cases}
    f\move(t,x,k),~t \in \T\plan \\
    f\brk(t,x,k),~t \in \T\brk \\
    f\stp(t,x,k),~t\in \T\stp.
\end{cases}
\end{align}
\end{defn}
\noindent Note this model is lower-dimensional than the high-fidelity model and generates \defemph{desired trajectories} in $X$.
The space $K \subset \R^{n_K}$ contains \defemph{trajectory parameters} that determine the ``shape'' of the desired trajectories.
We call these desired trajectories instead of reference trajectories to emphasize that the robot cannot track them perfectly.

Given a desired trajectory parameterized by $k \in K$, the robot uses a low-level controller $u_k: \T \times X\hi\times K \to U$ to track it.
Note that $u_k$ can be any sort of feedback controller, but typically cannot perfectly track the desired trajectories.
We say the robot ``tracks $k$'' to mean the robot tracks a desired trajectory parameterized by $k$.
When the robot tracks $k$, we predict its future state by applying $u_k$ as the input to the high-fidelity model \eqref{eq:high-fidelity_model}.

At the beginning of each planning iteration, time is reset to $t = 0$ and the origin of $X$ is translated and rotated to the robot's future pose, estimated as in Assumption \ref{ass:state_estimation_within_epsilon}.
During each planning iteration, we create a new desired trajectory for the next planning iteration by choosing $k \in K$ while tracking the previously-computed $k$.
Since $k$ does not change during each planning iteration, $\dot{k}(t) = 0$ for all $t \in \T$.

To simplify exposition, we do not show dependence on $k$ for $T\brk$ and $T\stp$ hereafter.
Note that $f\stp(t,x,k) = 0$ usually; we write $f\stp$ to illustrate that coming to a stop (i.e., completing the fail-safe maneuver) is part of every desired trajectory.
Since the robot cannot perfectly track trajectories produced by $f$, the stopped phase is included to ensure that the robot under $u_k$ comes to a complete stop.
Section \ref{sec:demos} describes an implementation of \eqref{eq:traj-prod_model_time_phases}.

\subsection{Tracking Error}

We can bound the spatial difference between the robot and the desired trajectory at any time; we call this the \defemph{tracking error}.
To construct this bound, we assume the following:

\begin{assum}\label{ass:sets_are_compact}
The spaces $X\hi$, $U$, and $K$ are compact.
The dynamics \eqref{eq:high-fidelity_model} is Lipschitz continuous in each of its arguments. 
\end{assum}

\begin{assum}\label{ass:tracking_error}
Let $i \in \{\mathrm{move, brake, stop}\}$ index the phases of $\T$ and let $j \in \{1, 2\}$ index the states in $X$.
Then, for each phase and state pair $(i,j)$, there exists a function $g_{i,j}: \T \times K \to \R_{\geq 0}$ such that $\mathrm{supp}\left(g_{i,j}\right) \subseteq \T_i\times K$ and for any $t \in T$ and $k \in K$ the following inequality holds:
\begin{align}\label{eq:g_tracking_error_defn}
  \max_{x\hio \in X\hio} \big|x\hij(t;x\hio,k) - x_j(t;x_0,k)\big| \leq \int_0^t \max_i \{ g_{i,j}(\tau,k) \}d\tau,
\end{align}
where $X\hio = \{x\hi \in X\hi~|~\idx(x\hi) \in X_0\}$,
$x\hij(t;x\hio,k)$ is the solution to \eqref{eq:high-fidelity_model} in state $j$ at time $t$ beginning from $x\hio$ under a control input $u_k$, and $x_j(t;x_0,k)$ is the solution to \eqref{eq:traj-prod_model_time_phases} in state $j$ at time $t$ beginning from $x(0) = \idx(x\hio)$ under a trajectory parameter $k$.
\end{assum}

\noindent We combine these $g_{i,j}$ to create the \defemph{tracking error function} $g: \T\times K \to (\R_{\geq 0})^2$, written as $g = (g_1,g_2)$, such that $g_j(t,k) = \max_i \{g_{i,j}(t,k)\}.$
As is proven in \cite[Lemma 12]{kousik2018_RTD_ijrr}, the tracking error function lets us ``match'' the spatial component of the high-fidelity model's trajectories using the trajectory-producing model.
\begin{lem}\label{lem:traj-tracking_model_matches_hi-fid_model}
For each $x\hio \in \{x\hi \in X\hi~|~\idx(x\hi) \in X_0\}$ and $k \in K$, there exists a $d \in L_d(\T)$ such that $\idx(x\hi(t;x\hio,k)) = \idx(x\hio) + \int_0^t\left(f(\tau,x(\tau;\idx(x\hio),k),k) + g(\tau,k)\circ d(\tau)\right)d\tau$
for each $t \in T$, where $x\hi(t;x\hio,k)$ is the solution to \eqref{eq:high-fidelity_model} at time $t$ beginning from $x\hio$ under a control input $u_k$ and $x(t;\idx(x\hio),k)$ is the solution to \eqref{eq:traj-prod_model_time_phases} at time $t$ beginning from $\idx(x\hio)$ under a trajectory parameter $k$.
\end{lem}
\noindent As shown further on in Lemma \ref{lem:w_i_geq_1_on_trajs}, this ``matching'' of spatial components lets us prove that the FRS for the lower-dimensional trajectory-producing model contains the behavior of the robot while it tracks the trajectory-producing model.
Note that the focus of this paper is not how to compute $g$, but rather how to use $g$ to conservatively approximate the behavior of the robot (Section \ref{sec:reachability}), which can then be used for online trajectory design (Sections \ref{sec:not-at-fault_plans} and \ref{sec:online_planning}).
Methods such as SOS optimization can be used to identify $g$ \cite[Chapter 7]{lasserre2009moments}.
\section{Dynamic Environments}\label{sec:dynamic_environments}

The mid-level planning method proposed in this paper generates desired trajectories for the robot to track in dynamic environments that ensure it is always not-at-fault.
We focus on ``not-at-fault'' behavior as opposed to ``safe'' behavior since there exist simple situations where no planner could ever guarantee collision-free behavior in the presence of malicious nearby actors.
Not-at-fault behavior requires sensing and predicting obstacles in the environment.
To provide any guarantees about the robot's behavior, we must ensure that it can sense all unoccluded obstacles that are within a certain distance of the robot.
This section first formalizes obstacles, predictions, and fault, then specifies a minimum sensor horizon to ensure that, while following plans generated by our mid-level planner, our robot, is always not-at-fault.

\subsection{Obstacles, Fault, and Predictions}

\begin{defn}\label{def:obs}
Given a time $t \geq 0$, an \defemph{obstacle} is a set in $X$ that the robot is not allowed to intersect with at time $t$.
Denote the $n^\regtext{th}$ obstacle at $t$ by $O_t^n\subset X$ for each $n \in \{1,\cdots,N\obs\}$.
\end{defn}

\noindent Using this definition, we can define not-at-fault behavior:
\begin{defn}\label{def:not-at-fault}
Let $t \geq 0$ be the current time.
If robot is moving at time $t$, it is \defemph{not-at-fault} if not intersecting any obstacle $O_t^n$.
If the robot is stationary at time $t$, it is always \defemph{not-at-fault}.
\end{defn}

\noindent By Definition \ref{def:not-at-fault}, a robot could be not-at-fault by staying stationary forever.
However, as we show in Section \ref{sec:demos}, the presented method is able to move the robot past obstacles while still being provably not-at-fault.
A more specific definition of fault could also be considered, such as one that required giving surrounding vehicles or agents enough space to brake to a stop or safely swerve away from our robot.
However this would require placing specific assumptions on how surrounding vehicles or agents respond to our motion (e.g. reaction time or rationality) \cite{shwartz2017_mobileye_safety}.
Under those assumptions, the presented method could potentially be adapted to more specific definitions of fault.

To generate not-at-fault plans, our planner must have access to a description of each obstacle's future behavior.

\begin{defn}\label{def:pred}
A \defemph{prediction} is a map $P_b: \T \to \P(X)$ that contains all obstacles within $\dl\sense$ (Assumption \ref{ass:obstacles_sensed_within_dlsense}) at each time $t' \in \T$; i.e., $P_b(t') \supseteq \bigcup_n O_{t'}^n$.
At each $t \in \T$, $P_b(t) \subseteq X$ is a union of a finite number of closed polygons, where the subscript denotes that the minimum distance between any obstacle and the boundary of $P_b$ is at least $b + \vep$ (Definition \ref{def:spatial_est_error}); i.e., for any $t \in \T$ and any obstacle $O_t^n$, $\inf\left\{\norm{p - q}_2~|~p \in \bd P_b(t),\ q \in O_t^n\right\} \geq b + \vep.$
\end{defn}
\noindent According to Definition \ref{def:pred} predictions must be \emph{correct} (all obstacles lie within the prediction at every time) and \emph{conservative} (the prediction overapproximates the obstacles at every time).
In addition, the difference between the state predicted by the high fidelity model and the true state of the robot over each planning time horizon is included in each prediction.
For convenience, we say that the prediction $P_b$ is \emph{buffered} by $b + \vep$.

Creating predictions that satisfy Assumption \ref{ass:obstacles_sensed_within_dlsense} and Definition \ref{def:pred} is the topic of ongoing research \cite{hoj2017_ped_pred,bajcsy2018_multi_robot_collision_avoidance,liu2017_dynamic_env_kinematic,shwartz2017_mobileye_safety}, but is not the focus of this work.
Instead, given such a prediction, we show how to design guaranteed not-at-fault trajectories.
To ensure such predictions could be generated, we place an assumption on the robot's sensor performance.
\begin{assum}\label{ass:obstacles_sensed_within_dlsense}
The robot senses all obstacles within a sensor radius $\dl\sense > 0$ and predicts their behavior.
\end{assum}
\noindent Note, during each planning iteration, the robot plans using the prediction generated at the beginning of that iteration.
To set a lower bound on the length of the sensing radius, we assume that there is a bounded number of obstacles sensed at each time and that the speed of any obstacle is finite:
\begin{assum}\label{ass:obs_number_and_speed}
There are up to $N\obsmax$ obstacles sensed at any time; i.e., $N\obs \leq N\obsmax$.
The speed of all obstacles is bounded by $v\obsmax \geq 0$.
\end{assum}
\noindent Occluded regions can be treated as dynamic obstacles \cite{yu2018occlusion} that can be conservatively predicted as moving at $v\obsmax$ in any direction, or can be subject to specific rules \cite{shwartz2017_mobileye_safety}.

\subsection{Minimum Sensor Horizon}

Per the discussion after Assumption \ref{ass:obstacles_sensed_within_dlsense}, the robot has to replan using the predictions available at the beginning of each planning iteration.
So, it must be able to sense obstacles that could cause a collision while it tracks a desired trajectory that begins at the \emph{end} of each planning iteration.
This means we must enforce a lower bound on the robot's sensor horizon so it detects obstacles from sufficiently far away.
This bound depends on how quickly the relative distance between our robot and any obstacle can change.
Recall that our robot has a maximum speed $\vmax$ by Assumption \ref{ass:speed} and obstacles have a maximum speed $v\obsmax$ by Definition \ref{def:obs}. The \defemph{maximum relative speed} between the robot and any obstacle is
\begin{align}
    v\rel = \vmax + v\obsmax.\label{eq:v_rel}    
\end{align}
Note that $v\rel$ ignores environmental constraints (e.g., traffic flow in lanes) that may reduce the maximum relative speed.
We now specify the minimum sensor horizon.

\begin{thm}\label{thm:sensor_horizon}
Let the current time be $0$ WLOG, and suppose the robot is tracking a not-at-fault desired trajectory for $t \in \T$.
Suppose the robot's sensor horizon is $\dl\sense \geq (\tfin + \tau\plan)v\rel + 2\vep,$
with $v\rel$ as in \eqref{eq:v_rel} and $\vep$ as in Definition \ref{def:spatial_est_error}.
Then, no obstacle whose points all lie farther than than $\dl\sense$ from the robot at the current time can cause a collision with the robot at any $t' \in [\tau\plan,\tau\plan + \tfin]$.
\end{thm}
\begin{proof}
While our robot executes the current desired trajectory for duration $\tau\plan$, only obstacles within a distance $\dl_1 = \tau\plan v\rel + \vep$ could cause a collision, by \eqref{eq:v_rel} and Definition \ref{def:spatial_est_error}.
As in Assumption \ref{ass:tau_plan_and_braking}, our robot either brakes and comes to a stop or tracks a new desired trajectory during $t \in [\tau\plan,\tau\plan + \tfin]$.
Then, again by \eqref{eq:v_rel} and Definition \ref{def:spatial_est_error}, only obstacles within at least $\dl_2 = \dl_1 + \tfin v\rel + \vep$ of the robot at time $t = 0$ could cause a collision when the robot tracks the new desired trajectory.
Since $\dl\sense \geq \dl_2$, the proof is complete.
\end{proof}
\section{Relating Predictions to Trajectories}\label{sec:reachability}

At each planning iteration, we want to select a desired trajectory that the robot can safely track.
Therefore, we want to compute a \defemph{not-at-fault map} $\NAFexact: \P(\T\times X) \to \P(K)$ from time and space (where predictions exist) to the trajectory parameters that, when tracked, guarantee the robot is not-at-fault.
Computing such a map requires understanding where the robot could be at any time while tracking any desired trajectory.
This section describes a method to compute an indicator function on the set of times and points that the robot could reach (i.e. the FRS) using SOS programming based on \cite{majumdar2014convex,shia2014convex,kousik2017safe}.
We construct a time-varying FRS since we are concerned with dynamic environments.
We incorporate the time phases of \eqref{eq:traj-prod_model_time_phases} by using a SOS program for each phase.
Finally, we conservatively approximate the not-at-fault map with the resulting indicator function on the FRS.
Note that the indicator function could also be computed with zonotopes or the level-set method (e.g., \cite{althoff2014online,mitchell2005time}), but a numerically certified way to compute the not-at-fault map has not yet been explored for those methods.

\subsection{The Forward Reachable Set}
The FRS contains all times and states reachable by the robot, described by \eqref{eq:high-fidelity_model}, when tracking any trajectory produced by \eqref{eq:traj-prod_model_time_phases}.
Note that the high-fidelity model \eqref{eq:high-fidelity_model} is typically of higher dimension than the FRS indicator functions that can be computed with SOS programming \cite{kousik2018_RTD_ijrr}.
However, by Lemma \ref{lem:traj-tracking_model_matches_hi-fid_model}, the trajectory-producing model and tracking error function can ``match'' any high-fidelity model trajectory on the space $\T\times X$.
This is useful because predictions exist in $\T\times X$.
We define the FRS of the trajectory producing model under disturbance as
\begin{align}\begin{split}\label{eq:FRS_overapprox}
    \FRS = \big\{(t,x) \in\ &\T\times X~|~\exists~(x_0,k) \in X_0\times K,~d \in L_d(\T)~\regtext{s.t.} \\
        &\dot{\tilde{x}}(\tau) = f(\tau,\tilde{x}(\tau),k) + g(\tau,k)\circ d(\tau) ~\forall~\tau \in \T, \\ &\tilde{x}(0) = x_0,~\regtext{and}~\tilde{x}(t) = x \big\}.
\end{split}\end{align}

\subsection{Computing the FRS}\label{subsec:reachability_computation}
Per \eqref{eq:traj-prod_model_time_phases}, the dynamics $f$ and tracking error $g$ are time-switched with three phases.
We therefore compute an outer approximation of $\FRS$ with the following sequence of three optimization programs, one for each phase.
First, we define the linear operators $\mathcal{L}_{f_i}, \mathcal{L}_{g_i}: C^1(\T\times X \times K) \to C(\T\times X\times K)$ given by $\mathcal{L}_{f_i} \phi(t,x,k) = \frac{d\phi}{dt}(t,x,k) + (\nabla_x\, \phi(t,x,k) \cdot f_i(t,x,k)$ and $\mathcal{L}_{g_i} \phi(t,x,k) = (\nabla_x\, \phi(t,x,k)) \cdot g_i(t,k)$.
Now let $i \in \{$move, brake, stop$\}$ and $t_{0,i} \in \{0,\ \tau\plan,\ \tau\plan+\tau\brk(k)\}$.
Then the following program, as we show in Lemma \ref{lem:w_i_geq_1_on_trajs}, constructs an outerapproximation to the indicator function on $\FRS$ in each $\T_i$:
\begin{flalign}
		& & \underset{v_i,w_i,q_i}{\text{inf}} \hspace*{0.25cm} & \int_{\T_i \times X \times K} w_i(t,x,k) ~ d\lambda_{\T_i \times X \times K} && \tag{$D_i$}\label{prog:find_FRS} \\
		& & \text{s.t.} \hspace*{0.25cm} & \Lfi v_i(t,x,k) + q_i(t,x,k) \leq 0, && \text{on }\T_i\times X\times K \nonumber\\
        & & & \Lgi v_i(t,x,k) + q_i(t,x,k) \geq 0,  && \text{on }\T_i\times X\times K \nonumber \\
        & & & -\Lgi v_i(t,x,k) + q_i(t,x,k) \geq 0, && \text{on }\T_i\times X\times K \nonumber \\
        & & & q_i(t,x,k) \geq 0,  && \text{on }\T_i\times X\times K \nonumber \\
        & & & -v_i(t_{0,i},x,k) \geq 0,  && \text{on } X_{i,0}\times K \nonumber \\
        & & & w_i(t,x,k) \geq 0, && \text{on } 
        \T_i \times X\times K \nonumber \\
        & & & w_i(t,x,k) + v_i(t,x,k) - 1 \geq 0, && \text{on }\T_i \times X\times K \nonumber,
\end{flalign}
where $v_i, w_i, q_i \in C(\T\times X\times K)$.
The space $X_{i,0}$ is the initial subset of $X$ the robot occupies in each mode at the time $t_{0,i}$ and is defined as follows: $X\moveo$ is the footprint of the robot, described in Assumption \ref{ass:planar_rigid_body}; $X\brko$ is the 0-level set of $v\move$ in $X\times K$ at the end of $\T\move$; and $X\stpo$ is the 0-level set of $v\brk$ in $X\times K$ at the end of $\T\brk$.
Next, by applying Theorem 4 from \cite{shia2014convex}, one can show that any feasible solution to \eqref{prog:find_FRS} overapproximates $\FRS$ in each phase $\T_i$:

\begin{lem}\label{lem:w_i_geq_1_on_trajs}
Let $(v_i,w_i,q_i)$ be a feasible solution to $(D_i)$ in phase $i$.
 Let $x\hi(t;x\hio,k)$ denote the solution to the high-fidelity \eqref{eq:high-fidelity_model} at time $t$ beginning from $x\hio \in \{x\hi \in X\hi~|~\idx(x\hi) \in X_0\}$ under control input $u_k$.
For every phase $i$, $t \in T_i$, $k \in K$, and $x\hio \in \{x\hi \in X\hi~|~\idx(x\hi) \in X_0\}$,
\begin{align}
    w_i(t,\idx(x\hi(t)),k) \geq 1.
\end{align}
\end{lem}

\subsection{Implementation}

We transform each \eqref{prog:find_FRS} into a semi-definite program (SDP) using SOS programming via the Spotless toolbox \cite{tobenkin2013spotless}, as covered in detail by \cite{shia2014convex,zhao2017optimal,majumdar2014convex}.
We solve the SDP with MOSEK \cite{mosek2010mosek}.
The key implementation difference is, where \cite{shia2014convex} and \cite{zhao2017optimal} solve a single SDP over multiple hybrid system modes, we solve a sequence of SDPs for each phase $\T_i$, with $i \in \{$move,brake,stop$\}$ and the initial condition sets $X_{i,0}$ implemented as discussed above.
For each $i$, the sequence of SDPs return $(v_i,w_i,q_i)$ as polynomials of fixed degree.
Note one can show that the solution to each SDP is a feasible solution to $(D_i)$ for each $i$ \cite[Theorem 6]{shia2014convex}. 
As a result, one can apply the result of Lemma \ref{lem:w_i_geq_1_on_trajs} to the solution of each SDP.

\subsection{The Not-at-Fault Map}

We conclude this section by conservatively approximating the not-at-fault map $\NAFexact$, beginning from the following observation:
To ease notation, extend the domain of each $w_i$ to $\T\times X \times K$ by setting $w_i(t,\cdot,\cdot) = 0~\forall~t \not\in T_i$ (we do not require $w_i$ to be differentiable on $\T\times X \times K$).
Then, we combine the $w_i$ into a single $w: \T\times X \times K \to \R$ as
\begin{align}
    w(t,x,k) = \max_i\{w_i(t,x,k)\}.
\end{align} 
By Lemma \ref{lem:w_i_geq_1_on_trajs}, $w \geq 1$ on trajectories of the high-fidelity model.
Using $w$, define $\NAF: \P(\T\times X) \to \P(K)$ as
\begin{align}\label{eq:not-at-fault_map}
    \NAF(T' \times X') = \{k\in K~|~w(t,x,k) < 1,\ t \in T',\ x \in X'\}.
\end{align}
It follows from Lemma \ref{lem:w_i_geq_1_on_trajs} that $\NAF$ underrapproximates $\NAFexact$ (meaning $k \in \NAF(t,x) \implies k \in \NAFexact(t,x)$).
Next, we use $\NAF$ to determine the not-at-fault parameters at each planning iteration.
\section{Not-at-Fault Plans}\label{sec:not-at-fault_plans}

\begin{figure}
    \centering
    \includegraphics[width=0.8\columnwidth]{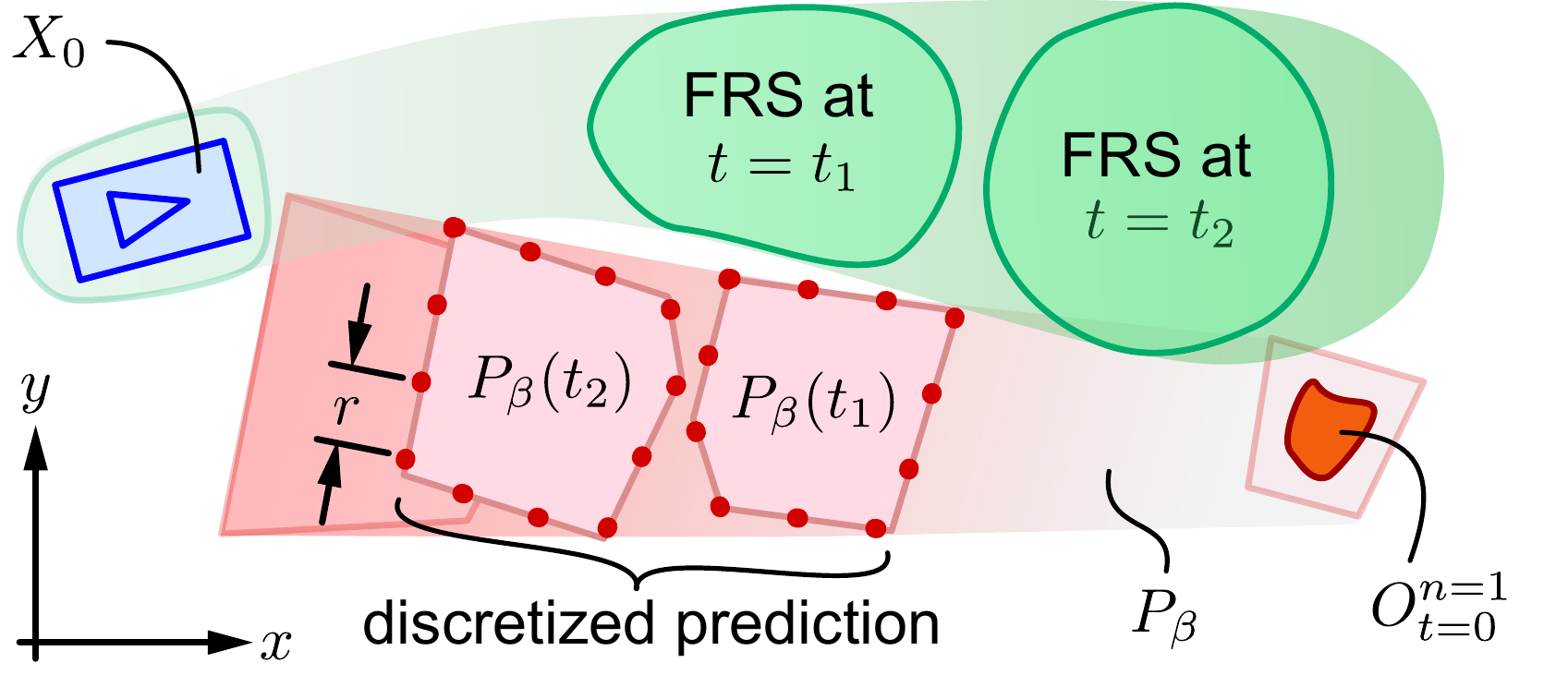}
    \caption{Discretization of a prediction $P_{\bt}$ as in Lemma \ref{lem:not-at-fault_over_short_interval}.
    Our robot plans a not-at-fault trajectory (FRS shown left to right) for any $t \in \T$ given the prediction (right to left).
    Temporal discretization is shown by two times, $t_1$ and $t_2$; at each time, the prediction is spatially discretized per Lemma \ref{lem:not-at-fault_at_t}.}
    \vspace*{-0.5cm}
    \label{fig:discretized_prediction}
\end{figure}

We now address how to find not-at-fault plans online.
This requires representing predictions in the trajectory parameter space in a way that enables real-time operation while guaranteeing not-at-fault behavior.
To understand how we construct this representation, consider a single planning iteration.
Suppose the robot generates a prediction $P_b$ for some $b > 0$ as in Definition \ref{def:pred} and Assumption \ref{ass:obstacles_sensed_within_dlsense}.
Then
\begin{align}
   K\naf = \NAF(P_b)
\end{align}
is a subset of trajectory parameters for which the robot is not-at-fault for all $t \in \T$.
Since the robot's sensor horizon $\dl\sense$ is as in Theorem \ref{thm:sensor_horizon}, if the robot executes the fail-safe maneuver from any $k \in K\naf$ for $t \geq \tau\plan$ and remains stopped thereafter, then it is not-at-fault for all time.
To choose a not-at-fault trajectory $k \in K\naf$, we must compute $K\naf$ at each planning iteration; it can be conservatively approximated with SOS programming, but doing so is intractable for real-time planning \cite[Section 6.1]{kousik2018_RTD_ijrr}.

This section presents a method to generate a subset of $K\naf$ by evaluating $\NAF$ on a discrete, finite subset of $\T\times X$.
This allows for optimizing over not-at-fault $k$ in real time using Algorithm \ref{alg:trajopt} in Section \ref{sec:online_planning}.
\citet{kousik2018_RTD_ijrr} prescribe a similar method to discretize obstacles while maintaining safety.
Unfortunately their technique is restricted to static obstacles.
We extend this method to incorporate predictions of dynamic obstacles.
Figure \ref{fig:discretized_prediction} illustrates the discretized prediction.

\begin{rem}\label{rem:robot_footprint_is_rectangle}
Throughout this section, we assume the robot has a rectangular footprint $X_0$ (as in Assumption \ref{ass:planar_rigid_body}) with width $W > 0$ and length $L > W$.
We extend the results to a circular robot footprint in Remark \ref{rem:b_and_r_for_circular_footprint}.
\end{rem}

\subsection{The Discretization Map}
Our goal is to discretize a prediction in space and time by introducing a \defemph{discretization map} that takes the graph of a prediction (in $\P(\T\times X)$) and returns a finite subset of the graph.
Example outputs of this map are illustrated in Figure \ref{fig:discretized_prediction}.
To construct this map, notice that by Definition \ref{def:pred}, at time $t$, the set $P_b(t) \subset X$ is the union of a finite set of closed polygons.
Therefore, the boundary $\bd P_b(t)$ can be written as a set $V \subset X$ of vertices and a set $E \subset X$ of edges \cite[Chapter 9.2]{minkowski_sum_fogel}.
We begin by sampling $\bd P_b (t)$ using the map $\sample: \P(X)\times \R_{\geq 0} \to X$. 
In particular, given $P_b(t)$ and a \defemph{point spacing} $r > 0$, let $\sample(P_b(t),r)$ return a (finite) set $A \subset X$ such that $V \subset A$ and such that for every point $a$ in $A$, there exists at least one distinct point $a' \in A$ such that $a' \in E$ and $\norm{a - a'}_2 \leq r$.
In other words, $\sample$ returns ``consecutive'' points around $\bd P_b(t)$ that are spaced no farther than $r$ apart.
We then define the discretization map $\discfn: \P(\T\times X)\times\P(\T)\times\R_{\geq 0} \to \P(\T\times X)$ as:
\begin{align}\begin{split}
    \discfn(P_b,T\disc,r) = \{(t,x) \in\ &T \times X ~|~t \in T\disc \text{ and } \\ &x \in \sample(P_b(t),r)\},
\end{split} \label{eq:disc_map}
\end{align}
where $P_b$ is the graph of $P_b(t)$ and $T\disc \subset T$.
In the remainder of this section, we show how to pick $T\disc$ and $r$ such that $\NAF(\discfn(P_b,T\disc,r)) \subseteq \NAF(P_b)$.

The following lemma, which is a direct application of \cite[Theorem 68]{kousik2018_RTD_ijrr}, illustrates how to pick $r > 0$ to discretize $P_b(t) \subset X$ such that the robot is not-at-fault at $t$.
This result requires Assumption \ref{ass:planar_rigid_body}, wherein the robot is a rigid body, and its footprint $X_0$ is compact and convex.
\begin{lem}\label{lem:not-at-fault_at_t}
\emph{(Not-at-fault at $t$)}
Pick a buffer distance $b \in (0, W/2)$, where $W$ is as in Remark \ref{rem:robot_footprint_is_rectangle}.
Let $P_b$ be a prediction as in Definition \ref{def:pred}, $t \in \T$, $t > 0$, $r = 2b$.
If the robot is not-at-fault for all $t' \in [0,t)$ then, while tracking $k \in \NAF(\discfn(P_b,\{t\},r))$, it is not-at-fault at time $t$.
\end{lem}

\noindent A point spacing $r$ that satisfies this lemma is illustrated in Figure \ref{fig:discretized_prediction}. 
Next, we create $T\disc \subset T$ such that ensuring safety at each $t \in T\disc$ is sufficient to ensure safety at each $t \in \T$.
To do so, we first explain how to pick a duration $\tau\disc > 0$ such that, if the robot is safe at a pair of times $t_1$ and $t_1 + \tau\disc$, it is safe for all $t \in [t_1,t_1+\tau\disc]$.
\begin{lem}\label{lem:not-at-fault_over_short_interval}
\emph{(Not-at-fault on a short interval)}
Pick $b \in (0,W/2)$ and a \defemph{temporal buffer} $b_t \in (0,\tfin\cdot v\rel/2)$, where $v\rel$ is as in \eqref{eq:v_rel}.
Let $\beta = b + b_t$ and suppose $P_{\bt}$ is a prediction.
Define the \defemph{maximum time discretization}:
\begin{align}
    \tau\discmax = (2b_t) / v\rel.\label{eq:time_discretization}
\end{align}
Suppose the current time is $t_1 \in [0, \tfin - \tau\discmax]$, $\tau\disc \in (0,\tau\discmax]$, and $t_2 = t_1 + \tau\disc$.
If the robot is not-at-fault for all $t \in [0,t_1)$, then it is not-at-fault over $[t_1,t_2]$ when tracking any $k \in \NAF\left(\discfn\left(P_{\bt},\{t_1,t_2\},r\right)\right)$.
\end{lem}
\begin{proof}
By Lemma \ref{lem:not-at-fault_at_t}, the closest that our robot can be to any obstacle at time $t_1$ is strictly greater than $b_t$ when tracking $k$.
Similarly, the closest it can be at time $t_2$ is strictly greater than $b_t$.
So, for the robot to collide with any obstacle over $[t_1, t_2]$, the robot must travel strictly more than $2b_t$ relative to the obstacle.
Therefore, the time difference between $t_1$ and $t_2$ must be less than or equal to $(2b_t)/v\rel =: \tau\discmax$.
Since $\tau\disc \leq \tau\discmax$ by construction, the relative distance that can be traveled over $[t_1, t_2]$ is less than or equal to $2b_t$.
\end{proof}

\noindent The time discretization of Lemma \ref{lem:not-at-fault_over_short_interval} is shown in Figure \ref{fig:discretized_prediction}.
We now ensure not-at-fault behavior for all time.

\begin{thm}\label{thm:not-at-fault_for_all_time}
\emph{(Not-at-fault for all time)}
Let $b, b_t, \bt, P_{\bt}$, and $r$ be as in Lemma \ref{lem:not-at-fault_over_short_interval}, $n\pred = \ceil{\tfin/\tau\discmax}$, $\tau\disc = \tfin/n\pred$, and
\begin{align}
    T\disc = \left\{j\cdot\tau\disc\right\}_{j = 0}^{n\pred} \label{eq:Tdisc}
\end{align}
Let $K_{\T} = \NAF\left(\discfn\left(P_{\bt},T\disc,r\right)\right)$.
If the robot is not at fault at $t = 0$, then it is not-at-fault for all $t \geq 0$ if it tracks any $k \in K_{\T}$ over $\T$ and remains stopped thereafter (i.e. $K\naf \supseteq K_{\T}$).
\end{thm}
\begin{proof}
Since the robot is not-at-fault at $t = 0$, 
by applying Lemma \ref{lem:not-at-fault_over_short_interval} at each $j\tau\disc$ for $j = 1,\cdots,n\pred$, the robot is not at fault for all $t \in \T$.
By \eqref{eq:traj-prod_model_time_phases} and Lemma \ref{lem:traj-tracking_model_matches_hi-fid_model}, the robot is stopped for all $t \geq \tfin$, so it is also not-at-fault by Definition \ref{def:not-at-fault}.
\end{proof}

\begin{rem}\label{rem:b_and_r_for_circular_footprint}
If the robot is circular instead of rectangular, with diameter $R$, pick $b \in (0,R/2)$ and set $r = 2R\sin\left(\cos\inv\left(\frac{R-b}{R}\right)\right)$.
Then, Lemma \ref{lem:not-at-fault_at_t}, Lemma \ref{lem:not-at-fault_over_short_interval}, and Theorem \ref{thm:not-at-fault_for_all_time} still hold \cite[Example 67]{kousik2018_RTD_ijrr}.
\end{rem}

\noindent Next, we use the discretized prediction for online planning.
\section{Online Planning}\label{sec:online_planning}

\begin{algorithm}[t]
\small
\begin{algorithmic}[1]
    \State {\bf Require:} $b$, $b_t$, $\NAF$, $T\disc$, $k_0 \in K$, $x\hio$, and $J: K \to \R$.
    
    \State {\bf Initialize:} $j = 0$, $t_j = 0$, $k^* = k_0$, $\bt = b + b_t$, $r = 2b$, $\xhij = x\hio$, \texttt{feas} = \texttt{true}.
    \State{\bf Loop:} // Line \ref{lin:control} executes at the same time as Lines \ref{lin:sense}--\ref{lin:predict}        
        
    \State\hspace{0.2in}{\bf Track} $k^*$ for $[t_j,t_j+\tau\plan)$\label{lin:control}
    
    \State\hspace{0.2in}$P_{\bt} \leftarrow \texttt{senseAndPredictObstacles}()$.\label{lin:sense}
    
    \State\hspace{0.2in}$D \leftarrow \discfn(P_{\bt},T\disc,r)$.\label{lin:discretize}
    
    \State\hspace{0.2in}{\bf Try} $k^* \leftarrow \regtext{argmin}_k\{J(k)~|~k\in\NAF(D)\}$ for duration $\tau\plan$\label{lin:trajopt}
    
    \State\hspace{0.2in}{\bf Catch} continue // $k^*$ is unchanged \label{lin:catch}
    
    \State\hspace{0.2in}$\xhijp \leftarrow \texttt{estimateFutureState}(t_j + \tau\plan, \xhij, k^*)$\label{lin:predict}
        
    \State\hspace{0.2in} $t_{j+1} \leftarrow t_j + \tau\plan$ and $j \leftarrow j + 1$\label{lin:increment}
    
    \State{\bf End}
\end{algorithmic}
\caption{\small RTD-D Online Planning}
\label{alg:trajopt}
\end{algorithm}

We now use the discretized prediction from Section \ref{sec:not-at-fault_plans} to plan online.
Assume the robot at $t = 0$ has a not-at-fault $k_0 \in K$.
Let $\NAF$ be as in \eqref{eq:not-at-fault_map}.
Pick $b$ and $b_t$ as in Lemma \ref{lem:not-at-fault_over_short_interval}, and 
Let $P_\bt$ be a prediction as in Definition \ref{def:pred}.
Let $J: K \to \R$ be an arbitrary cost function, such as a quadratic cost function that is minimized when the robot reaches a particular location.

Algorithm \ref{alg:trajopt} describes how RTD-D works online.
In each planning iteration, $\texttt{senseAndPredictObstacles}$ creates predictions as in Definition \ref{def:pred} (Line \ref{lin:sense}).
These obstacles are then discretized using \eqref{eq:disc_map} (Line \ref{lin:discretize}).
Then the planner attempts to find $k^*$ within $\tau\plan$ by optimizing over the user specified cost $J$ subject to satisfying the constraints (Line \ref{lin:trajopt}).
By the definition of $\NAF$ in \eqref{eq:not-at-fault_map}, the constraint in Line \ref{lin:trajopt} is equivalent to saying $w(t,x,k) < 1$ on any $(t,x)$ in the discretized prediction\footnote{Note, this constraint can be conservatively approximated as, e.g., $w \leq 0.999$, during implementation.}.
If $k^*$ is found within $\tau\plan$ in Line \ref{lin:trajopt}, it is tracked as in Assumption \ref{ass:tau_plan_and_braking} until a new $k^*$ is found; otherwise, the algorithm moves to Line \ref{lin:catch}, leaving $k^*$ unchanged.
On Line \ref{lin:predict}, \texttt{estimateFutureState} forward integrates \eqref{eq:high-fidelity_model} beginning at $x\hij$ under the control input $u_{k^*}$ (that tracks $k^*$) for a duration $\tau\plan$.
Concurrently, with each of these steps, the robot tracks the last feasible trajectory parameters it has constructed (Line \ref{lin:control}).
Note that we assume Lines \ref{lin:sense}, \ref{lin:discretize}, and \ref{lin:predict} happen instantaneously; however, in practice, the time to perform these steps can be subtracted from $\tau\plan$ to ensure satisfactory performance.
Finally, by applying Theorem \ref{thm:not-at-fault_for_all_time}, we can prove that RTD-D is not-at-fault for all time:

\begin{thm}\label{thm:alg_is_not-at-fault}
Suppose the robot's sensor horizon is as in Theorem \ref{thm:sensor_horizon}, the current time is $0$, and the robot has a not-at-fault $k_0 \in K$.
Then, by performing trajectory design and control using Algorithm \ref{alg:trajopt} with parameters as defined in Theorem \ref{thm:not-at-fault_for_all_time}, the robot is not-at-fault for all time.
\end{thm}
\section{Simulation and Hardware Demos}\label{sec:demos}

We demonstrate the proposed RTD-D method on two robot platforms in simulation and on hardware (shown in Figure \ref{fig:ev_and_segway}).

\subsection{Robots}
The first robot is a Segway RMP differential-drive robot with a high-fidelity model given by \cite[Example 7]{kousik2018_RTD_ijrr}.
The control inputs are desired yaw rate $u_1$ and desired speed $u_2$.
We find $\vep = 0.1$ m (as in Definition \ref{def:spatial_est_error}) and $c_1$ and $c_2$ from motion capture data.
The Segway has a circular footprint with radius $0.38$ m.
Mapping and localization are performed with a Hokuyo UTM-30LX lidar and Google Cartographer \cite{google_cartographer}.
RTD-D is run in MATLAB and ROS on a 4.0 GHz laptop.

\begin{figure}[t]
    \centering
    \includegraphics[height=33mm]{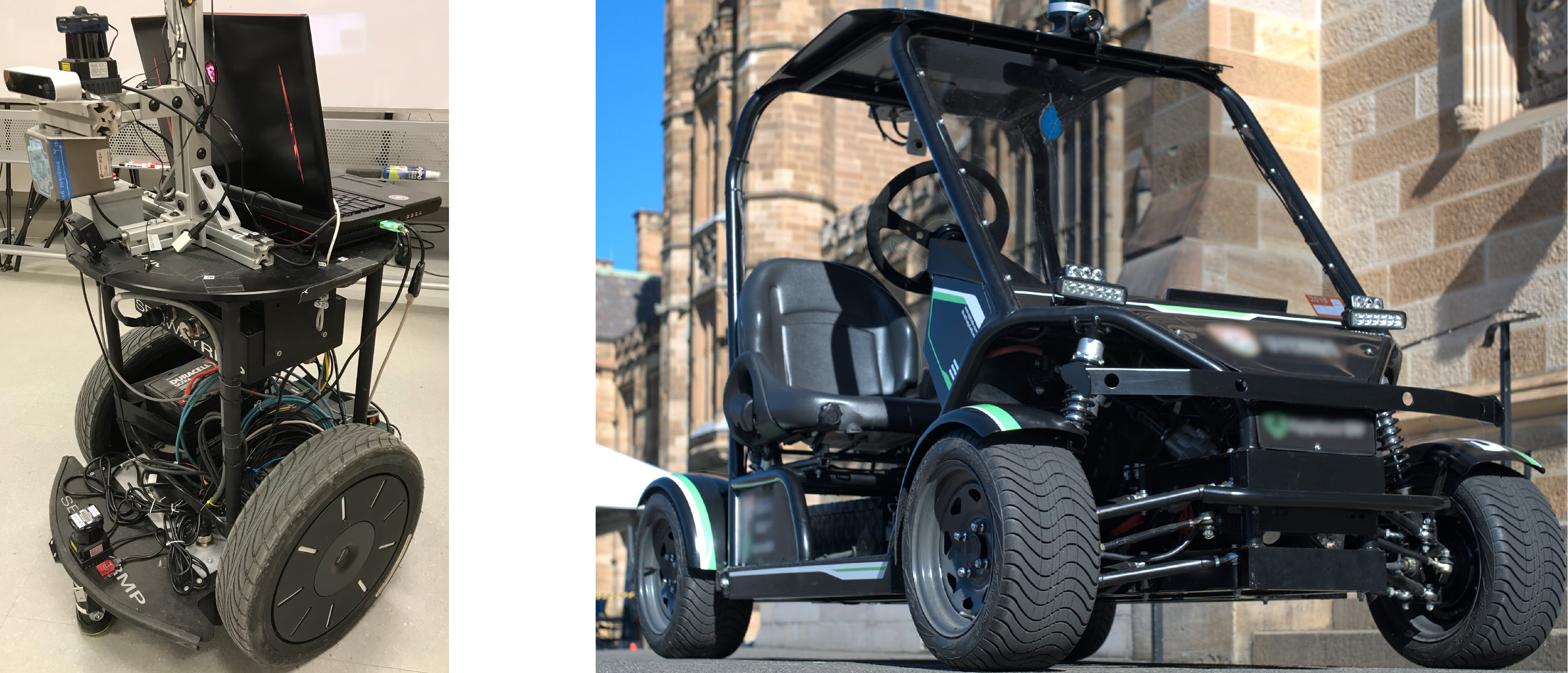}
    \caption{Robots used for simulation and hardware demos.
    The Segway is on the left and the EV on the right.}
    \label{fig:ev_and_segway}
    \vspace*{-0.5cm}
\end{figure}

The second robot is a small electric vehicle called the EV with the following high-fidelity model:
\begin{align}\label{eq:ev_hig_fid_model}
    \begin{bmatrix}\dot{x}_1(t)\\\dot{x}_2(t)\\\dot{\theta}(t)\\\dot{\delta}(t)\\\dot{v}(t)\end{bmatrix} &= \begin{bmatrix} 
        v(t)\cos(\theta(t))-\dot{\theta}(t)(c_1+c_2v(t)^2)\sin(\theta(t))\\
        v(t)\sin(\theta(t))+\dot{\theta}(t)(c_1+c_2v(t)^2)\cos(\theta(t))\\
        \tan(\delta(t))v(t)(c_3+c_4v(t)^2)^{-1} \\
        c_5(\delta(t)-u_1(t)) \\
        c_6+c_7(v(t)-u_2(t))+c_8(v(t)-u_2(t))^2
    \end{bmatrix},
\end{align}
where $\theta$ is heading, $\dl$ is steering angle, and $v$ is speed.
Saturation limits are $|\dl(t)| \leq 0.50$ rad, $|\dot{\dl}(t)| \leq 0.50$ rad/s, and $|\dot{v}(t)|\in\ [-6.86,3.50]$ m/s$^2$.
We find $\vep = 0.1$ m (as in Definition \ref{def:spatial_est_error}) and the coefficients $c_1,\cdots,c_8$ using localization data; the EV performs localization with a Robosense RS-Lidar-32 and saved maps \cite{berrio2018_robust_map}.
The EV has a rectangular $2.4\times 1.3$ m$^2$ footprint.
ROS runs on-board on a 2.6 GHz computer.
RTD-D is run in MATLAB on a 3.1 GHz laptop.
For both systems we implement Line \ref{lin:trajopt} in Algorithm \ref{alg:trajopt} using MATLAB's \texttt{fmincon}.

\subsection{RTD-D Implementation}\label{subsec:RTD_implementation}

Both robots create desired trajectories with:
\begin{align}\label{eq:example_traj-producing_model}
\begin{split}
f\move(t,x,k) &= \begin{bmatrix} k_2 \\ 0 \end{bmatrix} + \om\des(k)\begin{bmatrix}x_2 \\ x_1\end{bmatrix}\\
f\brk(t,x,k) &= s(t,k)\cdot f\move  (t,x,k) \\
f\stp(t,x,k) &= \begin{bmatrix}0\\0\end{bmatrix}.
\end{split}
\end{align}
This model produces circlar arcs that brake to a stop over $\T\brk$.
The trajectory parameters are $k = (k_1,k_2)$.
The desired yaw rate is $\om\des:  K \to \R$, given by $\om\des(k) = k_1$ for the Segway and $\om\des(k) = k_1k_2/\ell$ for the EV, where $\ell$ is the EV's wheelbase in meters.
For both robots, $k_2$ is desired speed.
The time phases of \eqref{eq:traj-prod_model_time_phases} are $\T\plan = [0,\tau\plan]$, $\T\brk = [\tau\plan,\tau\plan+\tau\brk(k)]$, and $\T\stp = [\tau\plan+\tau\brk(k),\tfin]$.
The braking time is $\tau\brk(k) = 1.0$ s for the Segway and $\tau\brk(k) = k_2/3$ for the EV.
We pick $\tfin$ by sampling the braking time for each robot's high-fidelity model.
The function $s: \T\times K \to \R$ is given by
\begin{align}
    s(t,k) = 1-\frac{t-\tau\plan}{\tau\brk(k)},\label{eq:time_slow_down_s_for_braking}
\end{align}
which slows the dynamics to zero over $\T\brk$.

For the Segway, $|k_1| \leq 1.5$ rad/s and $k_2 \in [0,2]$ m/s; between planning iterations, we limit commanded changes in $k_1$ (resp. $k_2$) to $0.5$ rad/s (resp. $0.5$ m/s).
For the EV, $|k_1| \leq 0.5$ rad and $k_2 \in [0,5]$ m/s; we limit commanded changes in $k_1$ (resp. $k_2$) to $0.1$ rad (resp. $0.5$ m/s).
For both robots, $u_k$ generates control inputs $u_{k,1}(t,k) = k_1$ and $u_{k,2}(t,k) = k_2$ $\forall~t\in \T\move$; $u_{k,1}(t,k) = s(t,k)k_1$ and $u_{k,2}(t,k) = s(t,k)k_2$ $\forall~t\in \T\brk$; and $u_{k,1}(t,k) = u_{k,2}(t,k) = 0~\forall~t\in \T\stp$.

To find tracking error functions, we simulate \eqref{eq:high-fidelity_model} under $u_k$ for each robot over a variety of initial conditions and desired trajectories.
We fit $g_{i,j}$ as polynomials satisfying Assumption \ref{ass:tracking_error}.
For each robot we solve \eqref{prog:find_FRS} as described in Section \ref{sec:reachability}, with $(v_i,w_i,q_i)$ as degree 10 polynomials.
For both robots we select the spatial and temporal buffers as $b = b_t = 0.1$ m.

\subsection{Simulation Demonstrations}\label{subsec:simulation_demonstrations}

For the Segway, simulations are in a $20\times10$ m$^2$ world with 1--10 $0.3\times0.3$ m$^2$ box-shaped obstacles.
We ran 100 trials for each number of obstacles (1000 trials total).
In each trial, a random start and goal are chosen approximately 18 m apart.
Each obstacle moves at a random constant speed along a random piecewise-linear path.
Simulations are identical for the EV, but the world is $60\times 10$ m$^2$, and the obstacles are $1\times 1$ m$^2$.
Both planners are restricted to $\tau\plan = 0.5$ s for each planning iteration (i.e., we require them to run in real time).
The spatial state estimation error is $\vep = 0$ for simulation.
At each planning iteration both planners are given a waypoint between the robot's position and the goal; the cost function for both planners is to reduce distance to the waypoint, resulting in optimizing to reach the global goal as fast as possible.

RTD-D is implemented as discussed above.
For comparison, a state lattice (SL) mid-level planner is implemented as in \cite{mcnaughton_thesis_2011} in MATLAB with braking as a fail-safe in each plan and LazySP for searching the lattice graph online \cite{dellin_and_srinivasa2016_lazysp}.
Similar to the approach used by \citet[Section 9.3.1]{kousik2018_RTD_ijrr}, SL was tested with obstacles buffered by increasing amounts until the planner had collisions in less than 10\% (resp. 20\%) of trials for the Segway (resp. EV); the final values were $0.43$ m (resp. $2.77$ m) for the Segway (resp. EV).
Since SL planners require feedback about the pose of the generated trajectories, we use a linear MPC controller for both robots.

Results are summarized in Table \ref{tab:simulation_results}.
Note, RTD-D has no at-fault collisions for either robot.
Collisions occur with SL because the robot cannot perfectly track its reference trajectory, and it is unclear how to buffer obstacles to provably compensate for tracking error and the robot's footprint (a variety of heuristics are presented in \cite{mcnaughton_thesis_2011}).
Compared to the Segway simulations, the EV simulations are more difficult because $v\rel$ is higher, leading to more collisions for SL.
Both planners stop more often than in the Segway simulations.
Note that the EV is not allowed to reverse and cannot turn in place, so it sometimes gets trapped by obstacles.

\begin{table}
\centering
\begin{tabular}{|c|c|r|r|r|r|}
\hline
Robot & Planner & \multicolumn{1}{c|}{AFC} & \multicolumn{1}{c|}{Goals} & \multicolumn{1}{c|}{AS} & \multicolumn{1}{c|}{APS} \\ \hline
 & \cellcolor[HTML]{EFEFEF}RTD-D & \cellcolor[HTML]{EFEFEF}\textbf{0.0} \% & \cellcolor[HTML]{EFEFEF}\textbf{100.0} \% & \cellcolor[HTML]{EFEFEF}1.17 m/s & \cellcolor[HTML]{EFEFEF}1.90 m/s \\ \cline{2-6} 
\multirow{-2}{*}{Segway} & SL & 7.6 \% & 92.4 \% & \textbf{1.37} m/s & \textbf{1.99} m/s \\ \hline
 & \cellcolor[HTML]{EFEFEF}RTD-D & \cellcolor[HTML]{EFEFEF}\textbf{0.0} \% & \cellcolor[HTML]{EFEFEF}\textbf{90.7} \% & \cellcolor[HTML]{EFEFEF}2.18 m/s & \cellcolor[HTML]{EFEFEF} \textbf{4.91 m/s} \\ \cline{2-6} 
\multirow{-2}{*}{EV} & SL & 17.2 \% & 77.3 \% &  \textbf{2.87 m/s} & 4.64 m/s \\ \hline
\end{tabular}
\caption{Simulation results for RTD-D versus a state lattice (SL) planner based on \cite{mcnaughton_thesis_2011}.
The ``AFC'' column is the percentage of trials with At-Fault Collisions as per Definition \ref{def:not-at-fault}.
RTD-D has no such collisions as expected, whereas the SL planner is not able to guarantee not-at-fault operation.
The ``AS'' column is Average Speed across jointly-successful trials (meaning trials in which both RTD-D and SL reached the goal).
Similarly, ``APS'' is Average Peak Speed across jointly-successful trials.
Using AS and APS as a measure of conservatism, we notice that RTD-D typically travels slightly slower than SL, but the tradeoff is worthwhile since RTD-D is always not-at-fault.}
\label{tab:simulation_results}
\vspace*{-0.5cm}
\end{table}

\subsection{Hardware Demonstrations}

To illustrate the capability of RTD-D, we also tested it on the Segway and EV hardware as described above, with videos available at \texttt{\url{www.roahmlab.com/ev\_dyn\_obs\_demo}} and (\texttt{\url{www.roahmlab.com/segway\_dyn\_obs\_demo}}.

The Segway runs indoors at up to 1.5 m/s in similar scenarios as in simulation.
Virtual dynamic obstacles ($v\obsmax = 1$ m/s) are created in MATLAB.
The testing area is smaller than the simulation world, so we only test with up to 3 obstacles.
The room boundaries are handled with RTD as in \cite{kousik2018_RTD_ijrr}.

The EV runs outdoors in a large open area at up to 3 m/s, with a safety driver.
For the EV we test more structured, car-like scenarios and show a variety of overtake maneuvers.
Virtual obstacles ($v\obsmax = 1.5$ m/s) resembling people or cyclists are created in MATLAB.
The area is large enough that we do not consider static obstacles.
\section{Conclusion}\label{sec:conclusion}

This paper introduces Reachability-based Trajectory Design for Dynamic environments (RTD-D), which generates provably not-at-fault, dynamically-feasible reference trajectories.
The contributions of this paper are: a minimum sensor horizon to ensure not-at-fault planning; a method for computing an FRS of a robot with tracking error and fail-safe maneuvers; an obstacle representation to guarantee choosing not-at-fault trajectories in real time; and successful simulation and hardware demonstrations of RTD-D.
For future work, we will extend this work to 3D systems and incorporate different types of uncertainty, such as varying road friction.

\renewcommand{\bibfont}{\normalfont\small}
{\renewcommand{\markboth}[2]{}
\printbibliography}

\end{document}